\documentclass[review,11pt]{elsarticle} % preprint

\usepackage{amsmath,epsfig}
\usepackage{amsthm, amssymb, stmaryrd}
\usepackage{pgfplots}
\usepackage{tikz}

\usepackage[ruled,vlined]{algorithm2e}
\usepackage{setspace}
\usepackage{float}
\usepackage{graphicx}
\usepackage{caption}
\usepackage{subcaption}
\usepackage{xcolor}
\usepackage{booktabs} % To thicken table lines

\newtheorem{theorem}{Theorem}[section]

\newtheorem{proposition}[theorem]{Proposition}

\newcommand{\obs}{{\boldsymbol{\mathrm{x}}_i^{(t)}}}
\newcommand*{\VEC}[1]  {\ensuremath{\boldsymbol{#1}}}
\newcommand*{\MAT}[1]  {\ensuremath{\boldsymbol{#1}}}
\def\Hzero{\mathrm{H}_0} % H0 hypothesis
\def\Hone{\mathrm{H}_1} % H1 hypothesis
\newcommand\supinf{\underset{\Hzero}{\overset{\Hone}{\gtrless}}} % Great Less for Detectors

\def\Tr{\,\text{tr}}
\DeclareMathOperator{\D}{D}

\DeclareMathOperator{\grad}{grad}

\journal{Signal Processing}

\begin{document}

\begin{frontmatter}

\title{Riemannian geometry for Compound Gaussian distributions: application to recursive change detection}

\author[label1]{Florent Bouchard}
\cortext[mycorrespondingauthor]{Corresponding author}
\ead{florent.bouchard@univ-smb.fr}
\author[label2]{Ammar Mian}
\author[label3]{Jialun Zhou}
\author[label3]{Salem Said}
\author[label1]{Guillaume Ginolhac}
\author[label3]{Yannick Berthoumieu}
\address[label1]{LISTIC, University Savoie Mont-Blanc, France}
\address[label2]{SONDRA, CentraleSupelec, France}
\address[label3]{IMS, University of Bordeaux, CNRS, France}

\markboth{Submission to Signal Process.,~2020}{Bouchard \MakeLowercase{\textit{et al.}}: New Geometric Distance for Compound Gaussian distributions. Application to Recursive Change Detection}

\begin{abstract}
A new Riemannian geometry for the Compound Gaussian distribution is proposed.
In particular, the Fisher information metric is obtained, along with corresponding geodesics and distance function.
This new geometry is applied on a change detection problem on Multivariate Image Times Series: a recursive approach based on Riemannian optimization is developed.
As shown on simulated data, it allows to reach optimal performance while being computationally more efficient.
\end{abstract}

\end{frontmatter}

%% INTRODUCTION
%%%%%%%%%%%%%%%%%%%%%%%%%%%%%%%%%%%%%%%%%%%%%%%%%%%%%%%%%%%%%%%%%%%%%%%%%%%%%%%%%%%%%%%%
\section{Introduction}
\label{sec:intro}

Covariance matrix is an important topic in signal and image processing.
When data are Gaussian distributed, the Maximum Likelihood Estimator (MLE) is the well known Sample Covariance Matrix (SCM).
However, this estimator features poor performance when data follow a more heavy-tailed distribution.
In such a case, it is interesting to model the data with a Complex Elliptically Symmetric (CES) distribution~\cite{Ollila2012} and to employ M-estimators~\cite{Pascal08} for covariance estimation.
In this paper, we limit ourselves to the Compound Gaussian (CG) distribution~\cite{Yao73,Gini2002}, which is a CES sub-family.
Its stochastic representation consists in a Gaussian vector multiplied by a positive scalar, called texture.
For instance, this family fits well RADAR empirical data~\cite{greco2007statistical}. 

It is possible to develop change detection algorithms for SAR Multivariate Image Times Series (MITS).
Several approaches exist and those based on a test of equality of covariance matrices generally perform well.
Moreover, they may have the interesting Constant False Alarm Rate (CFAR) property, in particular for Gaussian data~\cite{Conradsen2003,Novak2005,Ciuonzo2017}.
For the CG distribution, the Generalized Likelihood Ratio Test (GLRT) is derived in~\cite{MianTSP2019}.
This detector exhibits very good performance when data are not drawn from a Gaussian distribution.
However, when the number of images $T$ of the MITS is large, the computational time becomes prohibitive for practical implementation.
In this paper, a recursive implementation of this detector is proposed.
Because of the form of the change detector, this implementation cannot be derived easily, for example by employing an arithmetic mean.

To solve the problem, a framework based on a recursive approach as proposed in~\cite{ZS18} is developed and adapted to the CG distribution.
In order to do so, the Riemannian geometry of the CG distribution has to be considered, which, to the best of our knowledge, has not been done previously.
Hence, the main contribution of this paper consists in deriving a well-suited Riemannian geometry for the distribution of interest, \emph{i.e.} metric, geodesics, distance.
It relies on the Fisher information metric of the CG distribution; see \emph{e.g.}~\cite{S05,Breloy2019} for Gaussian and CES cases.
In addition, the Riemannian gradient to recursively estimate the CG parameters of a MITS and the corresponding Intrinsic Cram\'er Rao Bound (ICRB)~\cite{S05} are provided.
Finally, the proposed method is validated on simulated data.

% This paper is organized as follows.
% Section~\ref{sec:datamodel} presents the data model.
% In Section~\ref{sec:newgeodist}, the geometry is derived.
% Section~\ref{sec:RecCD} contains the application to change detection.
% In Section~\ref{sec:simu}, the proposed approach is validated on simulated~data.

%% Data model
%%%%%%%%%%%%%%%%%%%%%%%%%%%%%%%%%%%%%%%%%%%%%%%%%%%%%%%%%%%%%%%%%%%%%%%%%%%%%%%%%%%%%%%%
\section{Data Model}
\label{sec:datamodel}

Let a MITS $\{\mathbf{x}_i^{(t)}\}_{i\in\llbracket1,n\rrbracket,t\in\llbracket1,T\rrbracket}$ of $T$ data composed of $n$ samples in $\mathbb{C}^{p}$.
Even though these data follow the same statistical distribution, their parameters might change with $t$.
From this MITS, we want to detect these changes by comparing the parameters of the distribution, denoted $\theta^{(t)}$.
The change detection problem can be written as:
\begin{equation}
    \left\{
        \begin{array}{ll}
          \mathrm{H_0}: \theta^{(1)}=\theta^{(2)}=...=\theta^{(T)}=\theta^{(0)}\\
          \mathrm{H_1}: \exists (t,t')\in\llbracket 1,T\rrbracket^2, \theta^{(t)} \ne \theta^{(t^\prime)}\\
        \end{array}
    \right.
\end{equation}
As shown in~\cite{MianTSP2019}, to reach good performance, it is important that the parameters capture both the power and the correlations of the data.
To ensure this, we propose to use the CG distribution~\cite{Yao73,Gini2002} (also referred to as a mixture of scaled Gaussian).
This model corresponds to a Gaussian one, where each realization $\obs\in\mathbb{C}^p$ is scaled by a local power factor $\tau^{(t)}_i$ referred to as texture sample (assumed unknown deterministic in this work):
\begin{equation}
\obs \sim \mathcal{CN} (   \mathbf{0} , \tau^{(t)}_i \mathbf{\Sigma}^{(t)} )
\end{equation}
For the parameters to be identifiable, a constraint on the covariance $\mathbf{\Sigma}^{(t)}$ is needed.
Most often, a trace constraint $\Tr(\mathbf{\Sigma}^{(t)})=p$ is applied.
However, from a geometrical point of view, it is not the best choice.
In the following, we choose the unitary determinant normalization, advocated in~\cite{paindaveine2008canonical} because it allows to decorrelate the estimation of textures and covariance matrix.
In this paper, we further show that it yields tremendous simplifications in the Fisher information metric.
Thus, $\mathbf{\Sigma}^{(t)}$ belongs to
\begin{equation}
	\mathcal{SH}_{p}^{++} = \left\{ \mathbf{\Sigma}\in \mathcal{H}_{p}^{++}: \; |\mathbf{\Sigma}| = 1 \right\},
\label{eq:DefSubM}
\end{equation}
where $\mathcal{H}_{p}^{++}$ is the manifold of $p\times p$ positive definite matrices.

In~\cite{MianTSP2019}, the GLRT for the CG model is derived and the following detector is obtained:
\begin{equation}
	\hat{\Lambda}_{CG}^{(T)}  = \frac{\left|\hat{\mathbf{\Sigma}}^{(T)}_{0}\right|^{Tn}}{\displaystyle \prod_{t=1}^T \left| {\hat{\mathbf{\Sigma}}_{Tyl}^{(t)}}\right|^n} \displaystyle\prod_{\substack{i =1}}^{\substack{i=n}}  \frac{ \left(\displaystyle \sum_{t=1}^T  \hat{\tau}_{i,0}^{(t)} \right)^{Tp} }{\displaystyle \prod_{t=1}^T\left( \hat{\tau}_i^{(t)} \right)^{p}} \supinf \lambda ,
\label{eq : GLRT Mat Gen}
\end{equation}
where $\hat{\mathbf{\Sigma}}_{Tyl}^{(t)}$ and $\hat{\tau}_i^{(t)}$ are the classical Tyler's estimators of covariance and textures~\cite{Tyler1987,Pascal2008a}:
\begin{equation}
	\hat{\mathbf{\Sigma}}_{Tyl}^{(t)} = \frac{p}{n} \sum_{i=1}^n \frac{\obs \obs^H}{\obs^H (\hat{\mathbf{\Sigma}}_{Tyl}^{(t)})^{-1} \obs}
	\textup{\quad and\quad}
	\hat{\tau}_i^{(t)} = \frac{\obs^H (\hat{\mathbf{\Sigma}}_{Tyl}^{(t)})^{-1} \obs}{p};
	\label{eq:defTyl}
\end{equation} 
$\hat{\mathbf{\Sigma}}^{(T)}_{0}$ and $\hat{\tau}_{i,0}^{(t)}$ are the MLE of the covariance matrix and the textures under the null hypothesis $\mathrm{H_0}$:
\begin{equation}
	\hat{\mathbf{\Sigma}}^{(T)}_{0}  = \frac{p}{n} \sum_{i =1}^{n}  \frac{\displaystyle \sum_{t=1}^T \obs \obs^H}{\displaystyle \sum_{t=1}^T \obs^H (\hat{\mathbf{\Sigma}}_{0}^{(T)})^{-1} \obs}
	\mbox{\quad and\quad}
	\hat{\tau}_{i,0}^{(t)}=\frac{\obs^H (\hat{\mathbf{\Sigma}}^{(T)}_{0})^{-1}\obs}{Tp}.
	\label{eq:defTylerNullHypo}
\end{equation}
This detector features interesting CFAR properties and exhibits better performances when data follow a CG distribution.
Unfortunately, it suffers a large complexity, in particular as $T$ grows.
Moreover, when a new dataset $\{\mathbf{x}_i^{(T+1)}\}$ occurs, it is impossible to compute the new detector $\hat{\Lambda}_{CG}^{(T+1)}$ directly from $\hat{\Lambda}_{CG}^{(T)}$ because:
\begin{equation}
	\hat{\mathbf{\Sigma}}^{(T+1)}_{0} \neq \frac{T\hat{\mathbf{\Sigma}}^{(T)}_{0}+\hat{\mathbf{\Sigma}}_{Tyl}^{(T+1)}}{T+1}
\end{equation}
To avoid the computation of $\hat{\mathbf{\Sigma}}^{(T+1)}_{0}$ with all previous data, an original recursive approach based on Riemannian optimization is proposed.

%% New distance
%%%%%%%%%%%%%%%%%%%%%%%%%%%%%%%%%%%%%%%%%%%%%%%%%%%%%%%%%%%%%%%%%%%%%%%%%%%%%%%%%%%%%%%%
\section{Riemannian geometry of the compound Gaussian distribution}
\label{sec:newgeodist}

To simplify notations, the superscript $^{(t)}$ is omitted in this section.
In the following, $\boldsymbol{\tau}=[\tau_1\dots\tau_n]^T$, $\theta=(\mathbf{\Sigma},\boldsymbol{\tau})$, $\xi=(\mathbf{\xi}_{\mathbf{\Sigma}},\mathbf{\xi}_{\boldsymbol{\tau}})$ and $\eta=(\mathbf{\eta}_{\mathbf{\Sigma}},\mathbf{\eta}_{\boldsymbol{\tau}})$.
The parameter $\theta$ of the CG distribution lies in the manifold $\mathcal{M}_{p,n}=\mathcal{SH}^{++}_p\times\mathbb{R}_{++}^n$.
Since this is the product of two manifolds, $\mathcal{M}_{p,n}$ is also a manifold (see \emph{e.g.}~\cite{AMS08} for details).
Its tangent space $T_{\theta} \mathcal{M}_{p,n}$ at $\theta$ is $T_{\mathbf{\Sigma}}\mathcal{SH}^{++}_p\times T_{\boldsymbol{\tau}}\mathbb{R}_{++}^n$, where $T_{\mathbf{\Sigma}}\mathcal{SH}^{++}_p$:
\begin{equation}
	T_{\mathbf{\Sigma}}\mathcal{SH}^{++}_p = \{ \mathbf{\xi}_{\mathbf{\Sigma}}\in\mathcal{H}_p : \, \Tr(\mathbf{\Sigma}^{-1}\mathbf{\xi}_{\mathbf{\Sigma}})=0 \}
\label{eq:SHPD_tangent}
\end{equation}
($\mathcal{H}_p$ denotes the space of $p\times p$ Hermitian matrices); and $T_{\boldsymbol{\tau}}\mathbb{R}_{++}^n$ is identified to $\mathbb{R}^n$.

To turn $\mathcal{M}_{p,n}$ into a Riemannian manifold, it must be equiped with a Riemannian metric.
The most natural choice in our case is to consider the Fisher information metric on $\mathcal{M}_{p,n}$ associated with the CG distribution.
It is given in the following proposition.

\begin{proposition}[Fisher information metric]
\label{prop:fisher}
	The Fisher metric of the CG distribution on $\mathcal{M}_{p,n}$ is defined, for $\theta\in\mathcal{M}_{p,n}$ and $\xi,\eta\in T_{\theta}\mathcal{M}_{p,n}$, by, up to a factor,
	\begin{equation*}
		\langle \xi,\eta \rangle^{\mathcal{M}_{p,n}}_{\theta} = \frac1p\langle \mathbf{\xi}_{\mathbf{\Sigma}},\mathbf{\eta}_{\mathbf{\Sigma}} \rangle^{\mathcal{H}^{++}_p}_{\mathbf{\Sigma}} + \frac1n\langle \mathbf{\xi}_{\boldsymbol{\tau}} , \mathbf{\eta}_{\boldsymbol{\tau}} \rangle^{\mathbb{R}_{++}^n}_{\boldsymbol{\tau}},
	\end{equation*}
	with
	$\langle \mathbf{\xi}_{\mathbf{\Sigma}},\mathbf{\eta}_{\mathbf{\Sigma}} \rangle^{\mathcal{H}^{++}_p}_{\mathbf{\Sigma}} = \Tr (\mathbf{\Sigma}^{-1} \mathbf{\xi}_{\mathbf{\Sigma}} \mathbf{\Sigma}^{-1}\mathbf{\eta}_{\mathbf{\Sigma}})$
	and
	$\langle \mathbf{\xi}_{\boldsymbol{\tau}} , \mathbf{\eta}_{\boldsymbol{\tau}} \rangle^{\mathbb{R}_{++}^n}_{\boldsymbol{\tau}}=(\mathbf{\xi}_{\boldsymbol{\tau}}\odot\boldsymbol{\tau}^{\odot-1})^T \mathbf{\eta}_{\boldsymbol{\tau}}\odot\boldsymbol{\tau}^{\odot-1}$, where $\odot$ and $\cdot^{\odot-1}$ denote elementwise product and inversion, respectively.
\end{proposition}
\begin{proof}
The log-likelihood $L$ on $\mathcal{M}_{p,n}$ for $\theta$ is
\begin{equation}
	L_{\textup{CG}}(\theta) = \sum_i L_{\textup{G}}(\tau_i\mathbf{\Sigma}) = \sum_i L_{\textup{G}}\circ\varphi_i(\theta),
\label{eq:cG_log_likelihood}
\end{equation}
where $L_{\textup{G}}$ is the log-likelihood for the Gaussian distribution, see \emph{e.g.}~\cite{S05}; and $\varphi_i(\theta)=\tau_i\mathbf{\Sigma}$.
By definition and~\cite[Theorem 1]{S05},
\begin{equation*}
	\begin{array}{rcl}
		\langle \xi,\eta \rangle^{\mathcal{M}_{p,n}}_{\theta} & = & \mathbb{E}\left[ \D L_{\textup{CG}}(\theta)[\xi] \D L_{\textup{CG}}(\theta)[\eta] \right]
		= - \mathbb{E}\left[ \D^2 L_{\textup{CG}}(\theta)[\xi,\eta] \right] 
		\\[5pt]
		& = & - \sum_i \mathbb{E}\left[ \D^2 L_{\textup{G}}\circ\varphi_i(\theta)[\xi,\eta] \right]
		\\[5pt]
		& = & \sum_i  \mathbb{E}\left[ \D L_{\textup{G}}\circ\varphi_i(\theta)[\xi] \D L_{\textup{G}}\circ\varphi_i(\theta)[\eta] \right]
		\\[5pt]
		& = & \sum_i \langle \D\varphi_i(\theta)[\xi], \D\varphi_i(\theta)[\eta] \rangle^{\mathcal{H}^{++}_p}_{\varphi_i(\theta)},
	\end{array}
\end{equation*}
where $\D\varphi_i(\theta)[\xi]=\xi_{\boldsymbol{\tau}\,i}\mathbf{\Sigma}+\tau_i\mathbf{\xi}_{\mathbf{\Sigma}}$ is the directional derivative of $\varphi_i$.
Basic manipulations yield, up to a factor,
\begin{multline*}
	\langle \xi,\eta \rangle^{\mathcal{M}_{p,n}}_{\theta} =
	\frac1p\langle \mathbf{\xi}_{\mathbf{\Sigma}}, \mathbf{\eta}_{\mathbf{\Sigma}} \rangle^{\mathcal{H}^{++}_p}_{\mathbf{\Sigma}}
	+ \frac1n\langle \mathbf{\xi}_{\boldsymbol{\tau}} , \mathbf{\eta}_{\boldsymbol{\tau}} \rangle^{\mathbb{R}_{++}^n}_{\boldsymbol{\tau}}
	\\
	+ \frac1{np}\Tr(\mathbf{\Sigma}^{-1}\mathbf{\xi}_{\mathbf{\Sigma}})(\mathbf{\eta}_{\boldsymbol{\tau}}\odot\boldsymbol{\tau}^{-1})^T\mathbf{1}_n
	+ \frac1{np}\Tr(\mathbf{\Sigma}^{-1}\mathbf{\eta}_{\mathbf{\Sigma}})(\mathbf{\xi}_{\boldsymbol{\tau}}\odot\boldsymbol{\tau}^{-1})^T\mathbf{1}_n .
\end{multline*}
	Since $\mathbf{\xi}_{\mathbf{\Sigma}},\mathbf{\eta}_{\mathbf{\Sigma}}\in T_{\mathbf{\Sigma}}\mathcal{SH}^{++}_p$, we have $\Tr(\mathbf{\Sigma}^{-1}\mathbf{\xi}_{\mathbf{\Sigma}})=\Tr(\mathbf{\Sigma}^{-1}\mathbf{\eta}_{\mathbf{\Sigma}}) = 0$, which concludes the proof.
\end{proof}

% It follows that the orthogonal projection $P^{\mathcal{M}_{p,n}}_{\theta}$ from the ambient space $\mathbb{C}^{p\times p}\times\mathbb{R}^n$ onto $T_{\theta^{(t)}}\mathcal{M}_{p,n}$ according to $\langle\cdot,\cdot\rangle^{\mathcal{M}_{p,n}}_{\cdot}$ is
% \begin{equation}
% 	P^{\mathcal{M}_{p,n}}_{\theta}(\xi) = ( P^{\mathcal{SH}^{++}_p}_{\MAT{\Sigma}}(\MAT{\xi}_{\MAT{\Sigma}}), \VEC{\xi}_{\VEC{\tau}} ),
% \end{equation}
% where $P^{\mathcal{SH}^{++}_p}_{\MAT{\Sigma}}(\MAT{\xi}_{\MAT{\Sigma}}) = \mbox{herm}(\MAT{\xi}_{\MAT{\Sigma}}) - \frac{1}{p} \Tr(\mathbf{\Sigma}^{-1}\MAT{\xi}_{\MAT{\Sigma}})\mathbf{\Sigma}$ (see \emph{e.g.}~\cite{Breloy2019}).
% %
In the following proposition, the geodesics and Riemannian distance on $\mathcal{M}_{p,n}$ associated with the Fisher information metric $\langle\cdot,\cdot\rangle^{\mathcal{M}_{p,n}}_{\cdot}$ of the CG distribution are provided.
These geometrical objects are sufficient to perform Riemannian optimization and to measure and bound estimation errors.

\begin{proposition}[Geodesics and Riemannian distance]
\label{prop:geo_dist}
	The geodesic on $\mathcal{M}_{p,n}$ is $\gamma^{\mathcal{M}_{p,n}}(t) = ( \gamma^{\mathcal{SH}^{++}_p}(t), \gamma^{\mathbb{R}_{++}^n}(t) )$.
	If $\gamma^{\mathcal{M}_{p,n}}(0)=\theta$ and $\dot{\gamma}^{\mathcal{M}_{p,n}}(0)=\xi$,
	\[
		\gamma^{\mathcal{SH}^{++}_p}(t)=\mathbf{\Sigma}\exp(t\mathbf{\Sigma}^{-1}\VEC{\xi}_{\mathbf{\Sigma}})
		\textup{\quad and\quad}
		\gamma^{\mathbb{R}_{++}^n}(t) = \VEC{\tau}\odot\exp(t\VEC{\tau}^{\odot-1}\odot\VEC{\xi}_{\VEC{\tau}}).
	\]
	If $\gamma^{\mathcal{M}_{p,n}}(0)=\theta_0$ and $\gamma^{\mathcal{M}_{p,n}}(1)=\theta_1$,
	\[
		\gamma^{\mathcal{SH}^{++}_p}(t)=\MAT{\Sigma}_{0}^{1/2}(\MAT{\Sigma}_{0}^{-1/2}\MAT{\Sigma}_{1}\MAT{\Sigma}_{0}^{-1/2})^t\MAT{\Sigma}_{0}^{1/2}
		\textup{\quad and\quad}
		\gamma^{\mathbb{R}_{++}^n}(t)=\VEC{\tau}_{0}^{\odot1-t}\odot\VEC{\tau}_{1}^{\odot t}.
	\]
	It follows that the Riemannian distance on $\mathcal{M}_{p,n}$ corresponding to the Fisher metric of proposition~\ref{prop:fisher} is
	\begin{equation*}
		\delta_{\mathcal{M}_{p,n}}^2(\theta_0,\theta_1) = \frac1p\delta_{\mathcal{H}^{++}_p}^2(\MAT{\Sigma}_{0},\MAT{\Sigma}_{1}) + \frac1n\delta_{\mathbb{R}_{++}^n}^2(\VEC{\tau}_{0},\VEC{\tau}_{1}) ,
	% \label{eq:defdistance}
	\end{equation*}
where $\delta_{\mathcal{H}^{++}_p}^2(\MAT{\Sigma}_{0},\MAT{\Sigma}_{1})=\| \log(\MAT{\Sigma}_{0}^{-1/2}\MAT{\Sigma}_{1}\MAT{\Sigma}_{0}^{-1/2})\|^2_2$ and $\delta_{\mathbb{R}_{++}^n}^2(\VEC{\tau}_{0},\VEC{\tau}_{1})=\|\log(\VEC{\tau}_{0}^{-1}\odot\VEC{\tau}_{1})\|_2^2$.
\end{proposition}

\begin{proof}
The geodesics $\gamma^{\mathcal{SH}^{++}_p}(t)$ and $\gamma^{\mathbb{R}_{++}^n}(t)$ are the geodesics on $\mathcal{SH}^{++}_p$ and $\mathbb{R}_{++}^n$ equiped with $\langle\cdot,\cdot\rangle^{\mathcal{H}^{++}_p}_{\cdot}$ and $\langle\cdot,\cdot\rangle^{\mathbb{R}_{++}^n}_{\cdot}$, respectively.
Therefore, by definition of $\langle\cdot,\cdot\rangle^{\mathcal{M}_{p,n}}_{\cdot}$ and from the properties of product manifolds, $\gamma^{\mathcal{M}_{p,n}}$ is the geodesic on $\mathcal{M}_{p,n}$.
Similarly, $\delta_{\mathcal{H}^{++}_p}^2$ and $\delta_{\mathbb{R}_{++}^n}^2$ are the Riemannian distances associated with $\langle\cdot,\cdot\rangle^{\mathcal{H}^{++}_p}_{\cdot}$ and $\langle\cdot,\cdot\rangle^{\mathbb{R}_{++}^n}_{\cdot}$.
Thus, by definition of $\langle\cdot,\cdot\rangle^{\mathcal{M}_{p,n}}_{\cdot}$, $\delta_{\mathcal{M}_{p,n}}^2$ is the associated Riemannian distance on $\mathcal{M}_{p,n}$.
\end{proof}

%% Recursive CD
%%%%%%%%%%%%%%%%%%%%%%%%%%%%%%%%%%%%%%%%%%%%%%%%%%%%%%%%%%%%%%%%%%%%%%%%%%%%%%%%%%%%%%%%
\section{Application to recursive change detection}
\label{sec:RecCD}

Given a new data at $t+1$ $\{\mathbf{x}_i^{(t+1)}\}_i$, to obtain the CG change detector $\hat{\Lambda}_{CG}^{(t+1)}$ defined in~\eqref{eq : GLRT Mat Gen}, one needs to compute: $\hat{\theta}_{Tyl}^{(t+1)}=(\mathbf{\hat{\Sigma}}_{Tyl}^{(t+1)}, \boldsymbol{\hat{\tau}}_{Tyl}^{(t+1)})$ and $\hat{\theta}_{0}^{(t+1)}=(\mathbf{\hat{\Sigma}}_{0}^{(t+1)}, \boldsymbol{\hat{\tau}}_{0}^{(t+1)})$ defined in~\eqref{eq:defTyl} and~\eqref{eq:defTylerNullHypo}.
The complexity of the computation of $\hat{\theta}_{0}^{(t+1)}$ with usual techniques is quite high.
To solve this issue, a recursive implementation, obtained by exploiting the Riemannian derivation studied in~\cite{ZS18}, is proposed.
To estimate $\hat{\theta}_{0}^{(t+1)}$, we only use the information provided by $\hat{\theta}_{0}^{(t)}$ and the log-likelihood of the new data $\{\mathbf{x}_i^{(t+1)}\}_i$, which is
\begin{equation}
	L_{CG}^{(t+1)}(\theta) = \sum_i -p\log(\VEC{\tau}_i) - \frac{(\mathbf{x}_i^{(t+1)})^H\MAT{\Sigma}^{-1}\mathbf{x}_i^{(t+1)}}{\VEC{\tau}_i}.
\end{equation}
The recursive algorithm returning the sequence of estimates $\{\theta_0^{(t)}\}_t$ corresponding to the sequence of data $\{\mathbf{x}_i^{(t)}\}_{i,t}$ is given in Algorithm~\ref{algo:stepsize_decreasing}.
This algorithm relies on:
(\emph{i}) the Riemannian exponential map $\exp^{\mathcal{M}_{p,n}}_{\theta}:T_{\theta}\mathcal{M}_{p,n}\to\mathcal{M}_{p,n}$, such that $\exp^{\mathcal{M}_{p,n}}_{\theta}(\xi)=\gamma^{\mathcal{M}_{p,n}}(1)$, where $\gamma^{\mathcal{M}_{p,n}}$ is defined in Proposition~\ref{prop:geo_dist};
(\emph{ii}) the Riemannian gradient of $L_{CG}^{(t)}$, provided in Proposition~\ref{prop:grad}.

\begin{algorithm}
\KwIn{$\{\VEC{x}_i^{(t)}\}_{i,t}$, initialization $\theta^{(0)}\in\mathcal{M}_{p,n}$, initial stepsize $\alpha_0>0$}
\KwOut{$\{\theta^{(t)}\}_t$ in $\mathcal{M}_{p,n}$}
\For{$t=0$ \textbf{to} $T$}{
	$\theta^{(t+1)} = \exp^{\mathcal{M}_{p,n}}_{\theta^{(t)}}\left(\frac{\alpha_0}{t+1}\grad_{\mathcal{M}_{p,n}}L_{CG}^{(t+1)}(\theta^{(t)})\right)$
}
\caption{Recursive estimation of CG parameters in $\mathcal{M}_{p,n}$}
\label{algo:stepsize_decreasing}
\end{algorithm}

\begin{proposition}[Gradient of the parameters of CG distribution]
\label{prop:grad}
	The Riemannian gradient $\grad_{\mathcal{M}_{p,n}}L_{CG}^{(t)}(\theta)$ at $\theta\in\mathcal{M}_{p,n}$ is
	\begin{equation*}
		\grad_{\mathcal{M}_{p,n}}L_{CG}^{(t)}(\theta) = \left( \sum_i \frac{p\,\obs(\obs)^H - (\obs)^H\mathbf{\Sigma}^{-1}\obs \, \mathbf{\Sigma}}{\VEC{\tau}_i}, n(\VEC{a} - p\VEC{\tau}) \right)
	\end{equation*}
	where, for $1\leq i\leq n$, $\VEC{a}_i = (\obs)^H\mathbf{\Sigma}^{-1}\obs$.
\end{proposition}
\begin{proof}
By definition~\cite{AMS08}, for all $\xi\in T_{\theta}\mathcal{M}_{p,n}$,
$\langle \grad_{\mathcal{M}_{p,n}}L_{CG}^{(t)}(\theta), \xi \rangle^{\mathcal{M}_{p,n}}_{\theta} = \D L_{CG}^{(t)}(\theta)[\xi]$.
We have
\begin{equation*}
	\begin{array}{rcl}
		\D L_{CG}^{(t)}(\theta)[\xi]
		& = & \sum_i \frac{(\obs)^H\mathbf{\Sigma}^{-1}\obs - p\VEC{\tau}_i}{\VEC{\tau}_i^2}\VEC{\xi}_{\VEC{\tau}\,i}
		+ \frac{(\obs)^H\mathbf{\Sigma}^{-1}\MAT{\xi}_{\mathbf{\Sigma}}\mathbf{\Sigma}^{-1}\obs}{\VEC{\tau}_i}
		\\
		& = & \frac1n\langle n(\VEC{a} - p\VEC{\tau}), \VEC{\xi}_{\VEC{\tau}}\rangle^{\mathbb{R}_{++}^n}_{\VEC{\tau}}
		+ \frac1p\langle p\sum_i \frac{\obs(\obs)^H}{\VEC{\tau}_i}, \MAT{\xi}_{\mathbf{\Sigma}} \rangle^{\mathcal{H}^{++}_p}_{\mathbf{\Sigma}}.
	\end{array}
\end{equation*}
It remains to project $p\sum_i \frac{\obs(\obs)^H}{\VEC{\tau}_i}$ on the tangent space $T_{\mathbf{\Sigma}}\mathcal{SH}^{++}_{p}$.
This is achieved by using $P^{\mathcal{SH}^{++}_p}_{\MAT{\Sigma}}(\MAT{\xi}_{\MAT{\Sigma}}) = \mbox{herm}(\MAT{\xi}_{\MAT{\Sigma}}) - \frac{1}{p} \Tr(\mathbf{\Sigma}^{-1}\MAT{\xi}_{\MAT{\Sigma}})\mathbf{\Sigma}$ (see \emph{e.g.}~\cite{Breloy2019}).
One can check that it yields the proposed gradient.
\end{proof}

The Riemannian distance in Proposition~\ref{prop:geo_dist} can be used to measure the error contained in an unbiased estimator $\hat{\theta}^{(T)}$ of the parameter $\theta^{(T)}$ corresponding to a MITS with $T$ data.
Exploiting the same framework as in~\cite{S05,Breloy2019}, the corresponding ICRB is provided in the following proposition.

\begin{proposition}[ICRB]
	Given an unbiased estimator $\hat{\theta}^{(T)}$ of $\theta^{(T)}$ corresponding to a MITS with $T$ data, the ICRB corresponding to the error measured with the Riemannian distance in Proposition~\ref{prop:geo_dist} is
	\begin{equation*}
		\mathbb{E}[\delta_{\mathcal{M}_{p,n}}^2(\theta^{(T)},\hat{\theta}^{(T)})] \leq \frac{p^2-1+n}{Tpn}
	\end{equation*}
\end{proposition}
\begin{proof}
	By definition of $\langle\cdot,\cdot,\rangle^{\mathcal{M}_{p,n}}_{\cdot}$ in Proposition~\ref{prop:fisher}, the Fisher information matrix is $\MAT{F} = Tpn\MAT{I}_{p^2-1+n}$.
	Thus, $\Tr(\MAT{F}^{-1})=\frac{p^2-1+n}{Tpn}$, which is enough to conclude.
\end{proof}

%% SIMULATIONS
%%%%%%%%%%%%%%%%%%%%%%%%%%%%%%%%%%%%%%%%%%%%%%%%%%%%%%%%%%%%%%%%%%%%%%%%%%%%%%%%%%%%%%%%
\section{Numerical simulations}
\label{sec:simu}

Given $T$ data, the performance of the CG change detector~\eqref{eq : GLRT Mat Gen} under the null hypothesis greatly depends on the quality of the estimator $\theta_0^{(T)}$.
In this numerical experiment, we compare the performance of the three following estimators:
\begin{itemize}
	\item The MLE $\hat{\theta}_{mle}$, which features the best performance but is computationally expensive.
	\item The arithmetic mean $\hat{\theta}_{art}$, such that $\hat{\theta}^{(t+1)}_{art} = \frac{t\hat{\theta}^{(t)}_{art}+\hat{\mathbf{\theta}}^{(t+1)}_{Tyl}}{t+1}$, where $\hat{\mathbf{\theta}}^{(t+1)}_{Tyl}$ is Tyler's estimator~\eqref{eq:defTyl} of $\{\mathbf{x}_i^{(t+1)}\}_i$.
	\item The recursive estimation $\hat{\theta}_{rec}$ proposed in Algorithm~\ref{algo:stepsize_decreasing} with $\alpha_0=1/pn$.
\end{itemize}

% In this section, we will compare the performance of the CG change detector given in \eqref{eq : GLRT Mat Gen} under the null hypothesis for different kind of estimators $\theta^{(t)}=(\boldsymbol{\Sigma}_0^{(t)},\boldsymbol{\tau}_0^{(t)})$. Since this detection performance is linked to the quality of the estimators of $\theta^{(t)}$, we will compare the MSEs of these three following estimators:
% \begin{itemize}
% \item $\hat{\theta}_{mle}$ which is the MLE. This estimator will reach the best performance because it uses all the data for the estimation but its time computation increases with the number of images $T$.
% \item the arithmetic mean, $\hat{\theta}_{art}$, built as following:
% \begin{equation}
% \hat{\theta}^{(T+1)}_{art} = \frac{T\hat{\theta}^{(T)}_{art}+\hat{\mathbf{\theta}}^{(T+1)}_{Tyl}}{T+1}
% \end{equation}
% where $\hat{\mathbf{\theta}}^{(T+1)}_{Tyl}$ is obtained by computing the Tyler estimator \eqref{eq:defTyl} and the corresponding textures, both with the data $\{\mathbf{x}_i^{(T+1)}\}$.
% \item the recursive estimation, $\hat{\theta}_{rec}$ which is based on the algorithm \ref{algo:stepsize_decreasing} presented in the previous section.
% \end{itemize}

Simulated data $\{ \mathbf{x}_i^{(t)}\}_{i,t}$ of size $p=10$, $n\in\{20,50\}$, $T\in\llbracket1,1000\rrbracket$ are drawn from a $K$-distribution.
Textures $\VEC{\tau}$ follow a $\Gamma$ distribution with parameters $\alpha=\beta=1$.
The covariance matrix is generated as $\mathbf{\Sigma}=\mathbf{U}\mathbf{\Lambda}\mathbf{U}^H$, where $\mathbf{U}$ is a random unitary matrix drawn from a normal distribution and $\mathbf{\Lambda}$ is a random diagonal positive definite matrix with unitary determinant drawn from a chi-squared distirbution.

% The data $\{ \mathbf{x}_i^{(t)}\}$ of size $p=10$ are drawn from a $K$-distribution. In this distribution, the textures follow a $\Gamma$ distribution of parameters $(\alpha,\beta)$. Here we have chosen $\alpha=\beta=1$. To have the shape matrix, we build a matrix from an orthonomal matrix $\mathbf{U}$ and a diagonal matrix $\boldsymbol{\Lambda}$ (both drawn from a complex normal distribution). Finally, the shape matrix is obtained by: $\mathbf{U}\boldsymbol{\Lambda}\mathbf{U}^H/|\boldsymbol{\Lambda}|$. 

In Figure~\ref{fig:MSE}, we observe that, as expected, the MLE features the best performance and quicly reaches the ICRB as $T$ grows.
The arithmetic mean has good performance for small values of $T$ but reaches a minimal floor, thus displaying poor performance for large $T$.
Finally, our proposed method works quite well: it reaches the optimal performance as $T$ grows.
Moreover, it has the smallest complexity as only one iteration is needed for each new incoming data.

% The plots of Fig. \ref{fig:MSE} show the natural MSEs for the three estimators $\hat{\theta}_{mle}$, $\hat{\theta}_{art}$ and $\hat{\theta}_{rec}$ for two different values of $n$. For each plot, the ICRB \eqref{eq:ICRB} is also given. We first notice that the MLE quickly reaches the ICRB which is expected. The arithmetic mean has correct performance for small values of $T$ but at a point (different with respect to $n$), this MSE reaches a minimal floor. Finally, our proposed method works quite well since it has the same performance than the MLE from a value of $T$ (approximately 150 for $n=20$ and 50 for $n=50$). Moreover, it has the smaller time computation since only one iteration is needed for each new incoming data $\{ \mathbf{x}_i^{(t)}\}$.

\begin{figure*}[!t]
    \centering
    \begin{subfigure}{0.5\textwidth}
        \centering
\begin{tikzpicture}

\begin{axis}[%
width=0.75\textwidth,
height=0.75\textwidth,
at={(0,0)},
scale only axis,
xmode=log,
xmin=1,
xmax=1000,
xminorticks=true,
xlabel style={font=\color{white!15!black}},
xlabel={$T$},
ymin=-33,
ymax=0.6,
ylabel style={font=\color{white!15!black}},
ylabel={\footnotesize $\delta^2_{\mathcal{M}_{p,n}}(\theta,\hat{\theta})$ (dB)},
axis background/.style={fill=white},
title style={font=\bfseries},
title={$p=10$ and $n=20$},
xmajorgrids,
xminorgrids,
ymajorgrids,
legend style={legend cell align=left, align=left, draw=white!15!black},
legend style={at={(0.05,0.4)}, anchor=north west, draw=white!15!black, fill=white}
]
\addplot [color=blue, mark=asterisk, mark options={solid, blue}]
  table[row sep=crcr]{%
1	0.293145808881279\\
3	-6.26753154620261\\
10	-12.1169664589373\\
30	-16.6522672837745\\
100	-22.2387701465768\\
300	-26.7828248285721\\
1000	-32.045400099648\\
};
\addlegendentry{\tiny $\hat{\theta}_{mle}$}

\addplot [color=red, mark=square, mark options={solid, red}]
  table[row sep=crcr]{%
1	0.293135838776107\\
3	-5.24888025369078\\
10	-9.66439819550252\\
30	-11.2729146560004\\
100	-11.3741174060271\\
300	-10.984140625743\\
1000	-10.8137607101551\\
};
\addlegendentry{\tiny $\hat{\theta}_{art}$}

\addplot [color=green, mark=o, mark options={solid, green}]
  table[row sep=crcr]{%
1	0.293135838776107\\
3	-2.8208961223976\\
10	-8.59369267820261\\
30	-14.2701340714845\\
100	-20.9066184347642\\
300	-26.3002140018111\\
1000	-31.9008911625738\\
};
\addlegendentry{\tiny $\hat{\theta}_{rec}$}

\addplot [color=black]
  table[row sep=crcr]{%
1	-2.2548303427145\\
11	-12.6687571942968\\
21	-15.4770232900537\\
31	-17.1684472810572\\
41	-18.3826689099119\\
51	-19.3305321036939\\
61	-20.1081286928222\\
71	-20.7674138299053\\
81	-21.339680531501\\
91	-21.8452442659254\\
101	-22.2980440805409\\
111	-22.7080601305811\\
121	-23.082684045879\\
131	-23.4275432992721\\
141	-23.7470214692683\\
151	-24.0445998156462\\
161	-24.323089103033\\
171	-24.584791446636\\
181	-24.8316160914063\\
191	-25.0651640151918\\
201	-25.2867909169194\\
211	-25.4976548956914\\
221	-25.6987530795656\\
231	-25.8909501416359\\
241	-26.0750007684632\\
251	-26.2515675575249\\
261	-26.4212354160973\\
271	-26.5845232514586\\
281	-26.7418935417653\\
291	-26.8937602325736\\
301	-27.0404952986529\\
311	-27.1824342329829\\
321	-27.3198806667632\\
331	-27.4531102804717\\
341	-27.5823741326395\\
351	-27.7079015073727\\
361	-27.8299023617711\\
371	-27.948569438865\\
381	-28.0640800994707\\
391	-28.1765979166732\\
401	-28.2862740689163\\
411	-28.3932485614752\\
421	-28.4976513010712\\
431	-28.5996030443218\\
441	-28.6992162373929\\
451	-28.7965957614941\\
461	-28.891839596611\\
471	-28.9850394140035\\
481	-29.0762811064528\\
491	-29.1656452639442\\
501	-29.253207601387\\
511	-29.3390393440616\\
521	-29.4232075757098\\
531	-29.5057755535292\\
541	-29.5868029937802\\
551	-29.6663463312324\\
561	-29.7444589552761\\
571	-29.821191425173\\
581	-29.8965916666178\\
591	-29.9707051515271\\
601	-30.0435750627419\\
611	-30.11524244514\\
621	-30.1857463444803\\
631	-30.2551239351559\\
641	-30.3234106379027\\
651	-30.3906402283964\\
661	-30.4568449375709\\
671	-30.5220555444044\\
681	-30.5863014618424\\
691	-30.6496108164565\\
701	-30.7120105223811\\
711	-30.7735263500122\\
721	-30.8341829899088\\
731	-30.8940041122931\\
741	-30.9530124225078\\
751	-31.0112297127562\\
761	-31.0686769104202\\
771	-31.1253741232241\\
781	-31.1813406814875\\
791	-31.2365951776913\\
801	-31.2911555035569\\
811	-31.3450388848261\\
821	-31.3982619139089\\
831	-31.4508405805556\\
841	-31.5027903006936\\
851	-31.5541259435604\\
861	-31.6048618572511\\
871	-31.6550118927911\\
881	-31.704589426835\\
891	-31.7536073830832\\
901	-31.8020782525051\\
911	-31.8500141124445\\
921	-31.897426644683\\
931	-31.9443271525279\\
941	-31.9907265769871\\
951	-32.0366355120886\\
961	-32.0820642194\\
971	-32.1270226417946\\
981	-32.171520416514\\
991	-32.2155668875673\\
};
\addlegendentry{\tiny ICRB}

\end{axis}
\end{tikzpicture}%
    \end{subfigure}%
    \begin{subfigure}{0.5\textwidth}
        \centering
\begin{tikzpicture}

\begin{axis}[%
width=0.75\textwidth,
height=0.75\textwidth,
scale only axis,
xmode=log,
xmin=1,
xmax=1000,
xminorticks=true,
xlabel style={font=\color{white!15!black}},
xlabel={$T$},
ymin=-36,
ymax=-4,
ylabel style={font=\color{white!15!black}},
ylabel={\footnotesize $\delta^2_{\mathcal{M}_{p,n}}(\theta,\hat{\theta})$ (dB)},
axis background/.style={fill=white},
title style={font=\bfseries},
title={$p=10$ and $n=50$},
xmajorgrids,
xminorgrids,
ymajorgrids,
legend style={legend cell align=left, align=left, draw=white!15!black},
legend style={at={(0.05,0.4)}, anchor=north west, draw=white!15!black, fill=white}
]
\addplot [color=blue, mark=asterisk, mark options={solid, blue}]
  table[row sep=crcr]{%
1	-4.21213955259716\\
3	-9.72063505665841\\
10	-15.5523315405972\\
30	-19.9781974765823\\
100	-25.0988983735224\\
300	-29.7309463425868\\
1000	-35.2951219598952\\
};
\addlegendentry{\tiny $\hat{\theta}_{mle}$}

\addplot [color=red, mark=square, mark options={solid, red}]
  table[row sep=crcr]{%
1	-4.2121436341604\\
3	-9.3892453810683\\
10	-14.5137951795033\\
30	-18.7813682328672\\
100	-22.738762928187\\
300	-25.6193177061597\\
1000	-27.0989667081281\\
};
\addlegendentry{\tiny $\hat{\theta}_{art}$}

\addplot [color=green, mark=o, mark options={solid, green}]
  table[row sep=crcr]{%
1	-4.2121436341604\\
3	-8.67609928972059\\
10	-14.6743735612683\\
30	-19.572011945675\\
100	-24.8634195505881\\
300	-29.6864197423387\\
1000	-35.2887270139335\\
};
\addlegendentry{\tiny $\hat{\theta}_{rec}$}

\addplot [color=black]
  table[row sep=crcr]{%
1	-5.25783735923745\\
11	-15.6717642108197\\
21	-18.4800303065766\\
31	-20.1714542975802\\
41	-21.3856759264348\\
51	-22.3335391202168\\
61	-23.1111357093451\\
71	-23.7704208464282\\
81	-24.3426875480239\\
91	-24.8482512824484\\
101	-25.3010510970639\\
111	-25.711067147104\\
121	-26.0856910624019\\
131	-26.4305503157951\\
141	-26.7500284857912\\
151	-27.0476068321691\\
161	-27.3260961195559\\
171	-27.587798463159\\
181	-27.8346231079293\\
191	-28.0681710317147\\
201	-28.2897979334423\\
211	-28.5006619122144\\
221	-28.7017600960886\\
231	-28.8939571581589\\
241	-29.0780077849861\\
251	-29.2545745740478\\
261	-29.4242424326203\\
271	-29.5875302679815\\
281	-29.7449005582882\\
291	-29.8967672490965\\
301	-30.0435023151759\\
311	-30.1854412495058\\
321	-30.3228876832862\\
331	-30.4561172969946\\
341	-30.5853811491624\\
351	-30.7109085238957\\
361	-30.832909378294\\
371	-30.9515764553879\\
381	-31.0670871159936\\
391	-31.1796049331961\\
401	-31.2892810854393\\
411	-31.3962555779981\\
421	-31.5006583175941\\
431	-31.6026100608448\\
441	-31.7022232539158\\
451	-31.7996027780171\\
461	-31.8948466131339\\
471	-31.9880464305264\\
481	-32.0792881229758\\
491	-32.1686522804671\\
501	-32.2562146179099\\
511	-32.3420463605846\\
521	-32.4262145922327\\
531	-32.5087825700521\\
541	-32.5898100103031\\
551	-32.6693533477553\\
561	-32.7474659717991\\
571	-32.8241984416959\\
581	-32.8995986831408\\
591	-32.97371216805\\
601	-33.0465820792648\\
611	-33.118249461663\\
621	-33.1887533610033\\
631	-33.2581309516788\\
641	-33.3264176544256\\
651	-33.3936472449194\\
661	-33.4598519540939\\
671	-33.5250625609274\\
681	-33.5893084783653\\
691	-33.6526178329794\\
701	-33.715017538904\\
711	-33.7765333665351\\
721	-33.8371900064317\\
731	-33.897011128816\\
741	-33.9560194390307\\
751	-34.0142367292791\\
761	-34.0716839269432\\
771	-34.128381139747\\
781	-34.1843476980104\\
791	-34.2396021942142\\
801	-34.2941625200798\\
811	-34.348045901349\\
821	-34.4012689304319\\
831	-34.4538475970786\\
841	-34.5057973172166\\
851	-34.5571329600833\\
861	-34.607868873774\\
871	-34.6580189093141\\
881	-34.7075964433579\\
891	-34.7566143996062\\
901	-34.8050852690281\\
911	-34.8530211289674\\
921	-34.9004336612059\\
931	-34.9473341690509\\
941	-34.99373359351\\
951	-35.0396425286116\\
961	-35.0850712359229\\
971	-35.1300296583175\\
981	-35.1745274330369\\
991	-35.2185739040902\\
};
\addlegendentry{\tiny ICRB}

\end{axis}
\end{tikzpicture}%
    \end{subfigure}
    \caption{MSE $\delta_{\mathcal{M}_{p,n}}^2(\theta,\hat{\theta})$ as a function of $T$ with $p=10$, $n=20$ (left) and $n=50$~(right).}
    \label{fig:MSE}
\end{figure*}
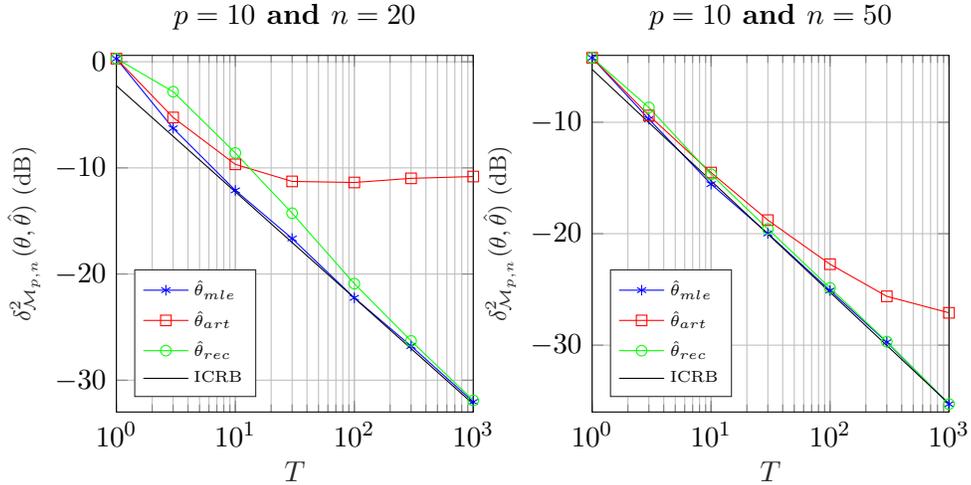

%% CONCLUSION
%%%%%%%%%%%%%%%%%%%%%%%%%%%%%%%%%%%%%%%%%%%%%%%%%%%%%%%%%%%%%%%%%%%%%%%%%%%%%%%%%%%%%%%%
\section{Conclusion}
\label{sec:conclude}

We have adapted a change detector derived for CG data in order to execute it recursively and greatly reduce the complexity of the calculation. This approach is based on Riemannian optimization which required the construction of geometry for CG distribution. Simulations have shown the interest of this new algorithm to reduce the complexity while maintaining good performance. 

%% APPENDIX
%%%%%%%%%%%%%%%%%%%%%%%%%%%%%%%%%%%%%%%%%%%%%%%%%%%%%%%%%%%%%%%%%%%%%%%%%%%%%%%%%%%%%%%%
\section*{Acknowledgment}
This work was supported by ANR PHOENIX (ANR-15-CE23-0012) and ANR-ASTRID MARGARITA (ANR-17-ASTR-0015).

%% BIBLIO
%%%%%%%%%%%%%%%%%%%%%%%%%%%%%%%%%%%%%%%%%%%%%%%%%%%%%%%%%%%%%%%%%%%%%%%%%%%%%%%%%%%%%%%%
\bibliographystyle{elsarticle-num}
\bibliography{BiblioRecCD}

\end{document}